\documentclass[runningheads]{llncs}
\usepackage[T1]{fontenc}

\usepackage{graphicx}
\usepackage{subcaption}

\usepackage{amsmath}
\usepackage{amssymb}
\usepackage{bm}
\usepackage{mathtools}

\usepackage{hyperref}
\usepackage{cleveref}
\usepackage{color}

\usepackage{listings}

\newcommand{\Prod}[3]{\left({#2}, {#3}\right)_{#1}}
\newcommand{\Prodh}[2]{\Prod{h}{#1}{#2}}
\newcommand{\Proddh}[2]{\Prod{\nabla, h}{#1}{#2}}

\newcommand{\Norm}[2]{\left\|{#2}\right\|_{#1}}
\newcommand{\Normh}[1]{\Norm{h}{#1}}
\newcommand{\Normdh}[1]{\Norm{\nabla, h}{#1}}

\newcommand{\Dp}{\nabla_{+}}
\newcommand{\Dm}{\nabla_{-}}
\newcommand{\Dxp}{\nabla_{x+}}
\newcommand{\Dxm}{\nabla_{x-}}
\newcommand{\Dyp}{\nabla_{y+}}
\newcommand{\Dym}{\nabla_{y-}}

\newcommand{\Dh}{D_h}
\newcommand{\Dzh}{D_{0,h}}
\newcommand{\Dph}{D^p_h}

\newcommand{\uh}{\bm{u}_h}
\newcommand{\vh}{\bm{v}_h}
\newcommand{\wh}{\bm{w}_h}
\newcommand{\sh}{\bm{\sigma}_h}
\renewcommand{\th}{\bm{\tau}_h}
\newcommand{\ph}{p_h}
\newcommand{\qh}{q_h}
\newcommand{\Uh}{\mathcal{U}_h}
\newcommand{\Vh}{\mathcal{V}_h}

\newcommand{\vhf}{\vh^{\bm{f}}}
\newcommand{\thf}{\th^{\bm{f}}}
\newcommand{\qhf}{\qh^{\bm{f}}}

\newcommand{\vhh}{\vh^{\bm{h}}}
\newcommand{\thh}{\th^{\bm{h}}}
\newcommand{\qhh}{\qh^{\bm{h}}}

\newcommand{\vhg}{\vh^{g}}
\newcommand{\thg}{\th^{g}}
\newcommand{\qhg}{\qh^{g}}

\newcommand{\Tp}[1]{\tau_{#1}^{+}}
\newcommand{\Tm}[1]{\tau_{#1}^{-}}
\newcommand{\Txp}{\Tp{x}}

\newcommand{\Txm}{\Tm{x}}

\begin{document}
\title{%
Python library supporting Discrete Variational Formulations and training solutions with Collocation-based Robust Variational Physics Informed Neural Networks (DVF-CRVPINN)
}
\titlerunning{Python library supporting discrete weak formulations with CRVPINN}
%
\author{%
Tomasz S\l{}u\.zalec\inst{1}\orcidID{0000-0001-6217-4274} \and
Marcin \L{}o\'s\inst{1}\orcidID{0000-0002-8426-6345} \and \\
Askold Vilkha\inst{1}\orcidID{0000-0001-9272-9082} \and
Maciej Paszy\'nski\inst{1}\orcidID{0000-0001-7766-6052}
}

\authorrunning{T. S\l{}u\.zalec et al.}

\institute{%
AGH University of Krakow, Poland \\
\email{\{los,sluzalec,maciej.paszynski\}@agh.edu.pl, askoldvilkha@gmail.com} }

\maketitle              
\begin{abstract}
We explore the possibility of solving Partial Differential Equations (PDEs) using discrete weak formulations. We propose a programming environment for defining a discrete computational domain, introducing discrete functions defined over a set of points, constructing discrete inner products, and introducing discrete weak formulations employing Kronecker delta test functions.
Building on this setup, we propose a discrete neural network representation, training the solution function defined over a discrete set of points and employing discrete finite difference derivatives in the automatic differentiation procedures. 
As a challenging computational model example, we focus on Stokes equations in two-dimensions, defined over a discrete set of points. We train the solution using the discrete weak residual and the Adamax algorithm with discrete automatic differentiation of the discrete gradients. Despite introducing the python environment, we also provide a rigorous mathematical formulation based on discrete weak formulations, proving the well-posedness and robustness of the loss function. The solution of the discrete weak formulations is based on neural network training employing a robust loss function that is related to the true error. In this way, we have a robust control of the numerical error during the training of the neural networks. Besides the Stokes formulation, we also explain the functionality of the proposed library using the Laplace problem formulation.
\keywords{%
Python library \and Robust Variational
Physics-Informed Neural Networks \and
Collocation Methods \and
Robust loss \and
Stokes Equations \and Laplace problem
}
\end{abstract}

\section{Introduction}

The rapid progress of deep learning over the last decade has had a profound impact on computational science, including the numerical solution of Partial Differential Equations (PDEs) (see, e.g., \cite{hinton2012deep}, \cite{10.1007/978-3-031-36021-3_52}, \cite{krizhevsky2017imagenet}, \cite{gheisari2017survey}). Among the most influential developments are Physics-Informed Neural Networks (PINNs) [\cite{raissi2019physics}], which embed physical constraints into neural network training by minimizing strong-form residuals evaluated at collocation points. 

PINNs have been applied to a wide range of problems, including fluid dynamics and Navier–Stokes equations \cite{cai2021physics}, \cite{mao2020physics}, \cite{WOS:000386452000009}, \cite{sun2020surrogate}, \cite{Wandel_Weinmann_Neidlin_Klein_2022}, wave propagation \cite{rasht2022physics}, \cite{maczuga2023influence}, \cite{WOS:000503737000017}, phase-field models \cite{WOS:000519657500012}, biomechanics \cite{WOS:000503006300001}, \cite{WOS:000496915700004}, quantum mechanics \cite{9891944}], electrical engineering \cite{9632308}, and inverse and uncertainty-aware problems \cite{IJCAI2},\cite{yang2019adversarial},\cite{WOS:000526518300074},\cite{Kim_Lee_Lee_Jhin_Park_2021},\cite{mishra2022estimates}. 

Variational Physics-Informed Neural Networks (VPINNs) \cite{kharazmi2019variational} were introduced as a weak-form alternative to classical PINNs, offering improved stability and accuracy. By constructing loss functions based on variational formulations, VPINNs have shown promising results for elliptic, parabolic, and transport-dominated problems \cite{kharazmi2021hp}, \cite{WOS:001021584500003}, \cite{WOS:001023419800001}, as well as inverse problems \cite{WOS:001047169000001}, \cite{WOS:001096752500001}. Nevertheless, the original VPINN formulation lacks robustness at the discrete level.

To address this issue, Robust VPINNs (RVPINNs) were proposed in \cite{rojas2024robust}, introducing Gram-matrix-weighted norms that ensure stability and consistency between discrete and continuous formulations. Despite their theoretical appeal, existing RVPINN implementations rely on numerical integration and dense linear algebra, leading to prohibitive computational costs and memory requirements.

The robust loss is defined as a Riesz representative of the residual, which practical implementation involves the inverse of the Gram matrix, namely,
\begin{equation}
\begin{aligned}
    \mathcal{L}(\theta) = \Norm{V'}{Au_\theta - f}^2 
    =R_\theta^T G R_\theta 
    =\text{RES}\,(\theta)^T G^{-1}\, \text{RES}(\theta)
\end{aligned}
\end{equation}
where
$R_\theta = G^{-1}\,\text{RES}(\theta)$ is a vector
    representing  $v \mapsto \Prodh{Au - f}{v}$, $\text{RES}(\theta)_{i} = \Prodh{Au - f}{v_i}$ is the weak residual vector,
    where $v_i$ is the~$i$-th basis function of~$V$, following the VPINN formulation.
This loss is robust, which means that the true error (the difference between the neural network solution and the exact solution) is bounded by the square root of the robust loss, up to the continuity and coercivity constants

\begin{equation}
    \frac{1}{\mu} \sqrt{\mathcal{L}(\theta)}
    \leq
    \Normh{u_\theta - u_\text{exact}}
    \leq
    \frac{1}{\alpha} \sqrt{\mathcal{L}(\theta)}.
\end{equation}

As a precursor to the present study, the collocation-based robust variational framework was introduced and validated in two spatial dimensions in \cite{los_collocation-based_2025}. In that work, the CRVPINN methodology combined point-collocation with a Gram-matrix-weighted loss function and an LU-based solver strategy, leading to a substantial reduction in computational cost compared to classical RVPINNs while preserving robustness. The approach was successfully demonstrated on a range of two-dimensional problems, including Laplace, advection–diffusion, Stokes, stationary Navier–Stokes, and linear elasticity equations, showing that the collocation-based formulation provides an efficient and stable alternative to integration-based variational PINNs.

In this paper, we introduce a programming framework for defining a discrete computational domain, incorporating discrete functions over a set of points, constructing discrete inner products, and formulating discrete weak problems using Kronecker delta test functions.
Building on this foundation, we develop a discrete neural network representation in which the solution function is defined over discrete points and trained using finite-difference-based derivatives within automatic differentiation.

As a challenging test case, we consider the two-dimensional Stokes equations on a discrete point set. The solution is obtained by minimizing a discrete weak residual using the Adam optimization algorithm, with gradients computed through discrete automatic differentiation.

In addition to the Python implementation, we provide a rigorous mathematical framework based on discrete weak formulations, establishing the well-posedness and robustness of the loss function. The neural network is trained using a loss function closely related to the true error, ensuring reliable control of numerical accuracy throughout training.

\begin{figure}
    \includegraphics[width=0.45\textwidth]{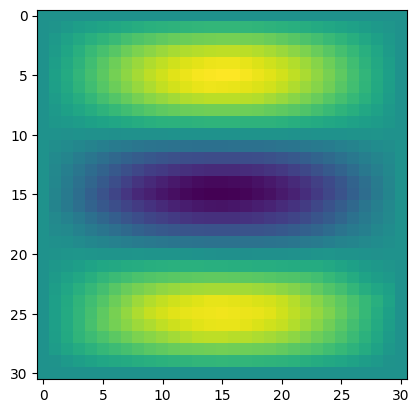}
    \includegraphics[width=0.55\textwidth]{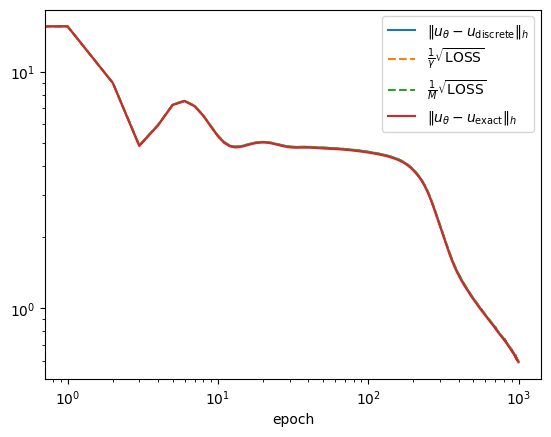}
    \caption{The solution of the Laplace problem with manufactured solution solved using our framework on  $30 \times 30$ grid. The Laplace problem has a continuity and coercivity constants equal to one, thus the convergence of trainig reviles the robust loss (equal to the true error). }
    \label{fig:pinn30x30}
\end{figure}

\section{Mathematical framework}

Following~\cite{los_collocation-based_2025},
we consider the 2D unit square~$\Omega = (0, 1)^2$
and a set of uniformly distributed collocation points
\begin{equation}
\Omega_{h} \coloneq \{ (ih, jh) : 0 \leq i, j \leq N \},
\end{equation}
where~$h = 1/N$ is the point spacing for some fixed grid size~$N$.
Our discrete formulations will be built using subspaces of the space
of real functions on~$\Omega_h$:
\begin{equation}
    D_h \coloneq \{ u \colon \Omega_h \to \mathbb{R} \} \cong \mathbb{R}^{(N + 1)^2}
\end{equation}
which we endow with a scalar product and the induced norm:
\begin{equation}
    \Prodh{u}{v} \coloneq h^2 \sum_{p \in \Omega_h} u(p)v(p),
    \qquad
    \Normh{u}^2 \coloneq \Prodh{u}{u}.
\end{equation}
Continuous derivatives will be approximated using finite difference operators
given by
\begin{equation}
\begin{aligned}
    \Dp u_{i,j} \coloneq & \left( \Dxp u_{i,j}, \Dyp u_{i,j} \right)
    \coloneq \left( \frac{u_{i+1, j} - u_{i,j}}{h}, \frac{u_{i, j+1} - u_{i,j}}{h}\right) \\
    \Dm u_{i,j} \coloneq & \left( \Dxm u_{i,j}, \Dym u_{i,j} \right)
    \coloneq \left( \frac{u_{i, j} - u_{i - 1,j}}{h}, \frac{u_{i, j} - u_{i,j - 1}}{h}\right)
\end{aligned}
\end{equation}
for these indices~$(i, j)$ where the right-hand side is defined,
and zero otherwise.
Here, $u_{i,j}$ is a shorthand for~$u(ih, jh)$.
Using these, we can define a degenerate scalar product and a corresponding seminorm:
\begin{equation}
    \Proddh{u}{v} \coloneq \Prodh{\Dxp u}{\Dxp v} + \Prodh{\Dyp u}{\Dyp v},
    \qquad
    \Normdh{u}^2 \coloneq \Proddh{u}{u}.
\end{equation}
When restricted to the subspace
\begin{equation}
    \Dzh \coloneq \left\{ u \in \Dh : \left.u\right|_{\partial\Omega_h} = 0 \right\},
    \quad
    \partial\Omega_h \coloneq \partial\Omega \cap \Omega_h
\end{equation}
of functions that vanish on the discrete boundary points,
these two become a scalar product and a norm, respectively.
The standard discrete Laplace operator
\begin{equation}
\Delta u_{i,j} \approx \frac{1}{h^2}\left(u_{i,j+1} + u_{i,j-1} + u_{i+1,j} + u_{i-1,j} -4u_{i,j}\right)
\end{equation}
corresponds to~$\Delta_h u = \Dp \cdot(\Dm u) = \Dm \cdot(\Dp u)$.


Let~$\tau_\alpha^{\pm}$ denote the shift operator in direction~$\alpha$,
that is
\begin{equation}
    \left(\Txp u\right)_{ij} =
    \begin{cases}
    u_{i+1,j} & \text{for }i < N \\
    0 & \text{for }i = N
    \end{cases}
    \quad
    \left(\Txm u\right)_{ij} = 
    \begin{cases}
    u_{i-1,j} & \text{for }i > 0 \\
    0 & \text{for }i = 0
    \end{cases},
\end{equation}
and similarly for~$\tau_y^\pm$ and~$\tau_z^\pm$.

\begin{lemma}
\label{lem:shifts-are-adjoin}
Shift operators~$\Tp{\alpha}$ and~$\Tm{\alpha}$ are adjoint,
that is~$\Prodh{\Tp{\alpha} u}{v} = \Prodh{u}{\Tm{\alpha} v}$
for all~$u, v \in \Dh$.
\end{lemma}

\begin{proof}
    We will prove it for~$\alpha = x$; for other directions, the proof is the same.
    Since the scalar product is bilinear, it is enough to consider the case
    when~$u$ is a basis function, 
    that is, $u = \delta^{ij}$ for some~$0 \leq i, j \leq N$.
    If~$i = 0$, then~$\Txp \delta^{ij} \equiv 0$, so~$\Prodh{\Txp \delta^{ij}}{v} = 0$.
    On the other hand, $\left(\Txm v\right)_{ij} = 0$,
    and consequently~$\Prodh{\delta^{ij}}{\Txp v} = 0$ as well.
    For~$i \neq 0$, function~$\Txp \delta^{ij}$ is nonzero only at the point
    corresponding to~$(i - 1, j)$,
    so
    \begin{equation}
        \Prodh{\Txp \delta^{ij}}{v} = h^3\, v_{i-1,j} = h^3\,\left(\Txm v\right)_{ij}
        = \Prodh{\delta^{ij}}{\Txm v}
    \end{equation}
    which completes the proof.
    \qed
\end{proof}

\begin{lemma}
\label{lem:shift-flips-gradient-polarity}
$\Tm{\alpha} \circ \Dp[\alpha] = \Dm[\alpha]$.
\end{lemma}

\begin{proof}
    Let~$u \in \Dh$, then for~$i > 0$ we have
    \begin{equation}
        \left(\Txm \Dxp u\right)_{ij} = \left(\Dxp u\right)_{i-1,j} = h^{-1}(u_{ij} - u_{i-1,j})
        = \left(\Dxm u\right)_{ij}.
    \end{equation}
    For~$i = 0$, both sides are zero.
    Proof for other directions is the same.
    \qed
\end{proof}

\begin{lemma}[Discrete product rule]
\label{lem:discrete-product-rule}
Any pair~$u,v \in \Dh$ satisfies
\begin{equation}
\begin{aligned}
  \Dp[\alpha] (u v) &= \left(\Dp[\alpha] u\right)\,v + \Tp{\alpha}u\left(\Dp[\alpha] v\right) \\
  \Dm[\alpha] (u v) &= \left(\Dm[\alpha] u\right)\,v + \Tm{\alpha}u\left(\Dm[\alpha] v\right).
\end{aligned}
\end{equation}
\end{lemma}

\begin{proof}
    For~$i < N$ we have
    \begin{equation}
    \begin{aligned}
        \Dxp(u v)_{ij} &= u_{i+1,j} v_{i+1,j} - u_{ij}v_{ij} \\&=
        u_{i+1,j}v_{i+1,j} - u_{i+1,j}v_{ij} + u_{i+1,j}v_{ij} - u_{ij}v_{ij} \\&=
        u_{i+1,j}(\nabla_{x+}v)_{ij} + (\nabla_{x+}u)_{ij}v_{ij} \\&=
        \left(\Txp u\right)_{ij}(\nabla_{x+}v)_{ij} + (\nabla_{x+}u)_{ij}v_{ij}.
    \end{aligned}
    \end{equation}
    For~$i = N$, both sides are zero.
    Proofs for other directions and for~$\Dm[\alpha]$ are the same.
    \qed
\end{proof}

\begin{lemma}
\label{lem:zero-gradient-integral}
Any~$u \in \Dzh$ satisfies~$\Prodh{\Dp[\alpha] u}{1} = 0$.
\end{lemma}

\begin{proof}
    \begin{equation}
    \begin{aligned}
        \Prodh{\Dxp u}{1} &= h^3 \sum_{\bm{x} \in \Omega_h} (\Dxp u)(\bm{x}) \\
        &= h^2\sum_{j,k = 0}^N \left\{ \sum_{i = 0}^N \left[u_{i+1,j,k} - u_{ijk}\right] \right\} \\
        &= h^2\sum_{j,k = 0}^N \left[ u_{Njk} - u_{0jk} \right] = 0,
    \end{aligned}
    \end{equation}
    since~$\left.u\right|_{\partial\Omega_h} = 0$.
    \qed
\end{proof}

\begin{lemma}[Discrete integration by parts]
\label{lem:discrete-by-parts}
Any pair~$u,v\in \Dh$ such that~$uv \in \Dzh$ satisfies
\begin{equation}
    \Prodh{\Dp[\alpha] u}{v} = - \Prodh{u}{\Dm[\alpha] v}
\end{equation}
\end{lemma}

\begin{proof}
    By~\cref{lem:zero-gradient-integral}, $\Prodh{\Dp[\alpha](uv)}{1} = 0$.
    Combining it with~\cref{lem:discrete-product-rule} we obtain
    \begin{equation}
        \Prodh{\Dp[\alpha] u}{v} + \Prodh{\Tp{\alpha}u}{\Dp[\alpha] v} =
        \Prodh{\Dp[\alpha](uv)}{1} = 0
    \end{equation}
    so~$\Prodh{\Dp[\alpha] u}{v} = - \Prodh{\Tp{\alpha}u}{\Dp[\alpha] v}$.
    By~\cref{lem:shifts-are-adjoin} we get
    $\Prodh{\Tp{\alpha}u}{\Dp[\alpha] v} = \Prodh{u}{\Tm{\alpha}\Dp[\alpha] v}$,
    and finally~\cref{lem:shift-flips-gradient-polarity}
    gives us~$\Prodh{u}{\Tm{\alpha}\Dp[\alpha] v} = \Prodh{u}{\Dm[\alpha] v}$,
    which proves the desired equality.
    \qed
\end{proof}

\begin{theorem}[Discrete Poincaré inequality]
\label{thm:discrete-poincare}
    There exists a constant~$K > 0$,
    independent of~$N$,
    such that
    \begin{equation}
        \Normh{\uh} \leq K \Normdh{\uh}
    \end{equation}
    for all~$\uh \in \Dzh$.
\end{theorem}

\begin{proof}
    Let~$g \in \Dh$ be given by~$g_{ij} = ih$.
    Then~$\left(\Dxm g\right)_{ij} = 1$ unless~$i = N$.
    Therefore, by~\cref{lem:discrete-by-parts}
    \begin{equation}
        \Normh{u}^2 = \Prodh{\Dxm g}{u^2} = \Prodh{g}{\Dxp (u^2)}.
    \end{equation}
    Using~\cref{lem:discrete-product-rule} we get
    \begin{equation}
        \Dxp u^2 = (u + \Txp u) \Dxp u.
    \end{equation}
    Since~$\Normh{\Txp u} = \Normh{u}$,
    using Cauchy-Schwarz inequality we obtain
    \begin{equation}
        \Normh{u}^2 = \Prodh{g}{(u + \Txp u) \Dxp u} \leq 2 \Normh{g} \Normh{u} \Normh{\Dxp u}
    \end{equation}
    and thus~$\Normh{u} \leq 2 \Normh{g} \Normh{\Dxp} \leq 2 \Normdh{u}$,
    since~$\Normh{g} \leq \Normh{1} = 1$.
    \qed
\end{proof}

\section{Theoretical foundations of the discrete weak formulation of the Stokes equations}
\label{sec:stokes-formulation}

We start by recalling the strong formulation of the Stokes equations:
find velocity~$\bm{u} \coloneq (u_x, u_y)$ and pressure~$p$ such that
\begin{equation}
\left\{
\begin{aligned}
    -\Delta\bm{u} + \nabla p &= \bm{f} \text{ in }\Omega\\
    \nabla\cdot\bm{u} &= 0 \text{ in }\Omega\\
    \bm{u} &= \bm{g} \text{ on } \partial\Omega
\end{aligned}
\right.
\end{equation}
By introducing stress~$\bm{\sigma}$ we can reformulate it as a first-order system:
\begin{equation}
\left\{
\begin{aligned}
    -\nabla\cdot\bm{\sigma} + \nabla p &= \bm{f} \text{ in }\Omega\\
    \nabla\cdot\bm{u} &= 0 \text{ in }\Omega\\
    \bm{\sigma} - \nabla \bm{u} &= 0 \text{ in }\Omega \\
    \bm{u} &= \bm{g} \text{ on } \partial\Omega
\end{aligned}
\right.
\end{equation}
As demonstrated in~\cite{roberts_dpg_2014},
if we define~$A \colon H_A \to \bm{L}^2(\Omega)$ by
\begin{equation}
    A(\bm{u}, \bm{\sigma}, p) \coloneq
    \left(
        -\nabla\cdot\bm{\sigma} + \nabla p,
        \nabla\cdot\bm{u},
        \bm{\sigma} - \nabla \bm{u}
    \right)
\end{equation}
where~$H_A \subset \bm{L}^2(\Omega) \times \bm{L}^2(\Omega) \times L^2(\Omega)$
is an appropriate Sobolev-type space,
the variational formulation~$\Prod{}{Au}{v} = \Prod{}{f}{v}$
is well-posed with~$L^2$ norm on the test space
and the graph norm~$\Norm{H_A}{u}^2 \coloneq \Norm{L^2}{u}^2 + \Norm{L^2}{Au}^2$
on the trial space.
Furthermore, to obtain stability in the~$L^2$ norm on the trial space,
it may be possible to use the
\emph{graph adjoint norm}~$\Norm{H_{A^*}}{v}^2 \coloneq \Norm{L^2}{v}^2 + \Norm{L^2}{A^*v}^2$
on the test space, where~$A^*$ is the formal adjoint of~$A$.
We apply these ideas to derive a stable discrete formulation
of the Stokes equations in the CRVPINN framework.

To ensure our discrete formulation is well-defined,
we must take care to properly choose the discrete space
where~$u_h = (\uh, \sh, \ph)$ lives.
Let us denote it by~$\Uh = V_h \times S_h \times P_h$
and endow it with the~$\Dh$ scalar product, that is,
for~$u_h = (\uh, \sh, \ph)$
and~$v_h = (\vh, \th, \qh)$ we let
\begin{equation}
    \Prod{\Uh}{u}{v} = \Prodh{\uh}{\vh} + \Prodh{\sh}{\th} + \Prodh{\ph}{\qh}
\end{equation}

For the velocity space~$V_h$, the natural choice is~$\Dzh^2$,
since it incorporates the Dirichlet boundary conditions.
In a few places we will use a projection operator~$\pi_0 \colon \Dh^2 \to \Dzh$
that sets the boundary values of a vector field to zero.
We can replace the Laplacian with its discrete version~$\Delta_h$.
For the first order formulation, we can write it as 
either~$\Delta_h \uh = \Dp \cdot \sh$,
with~$\sh = \Dm \uh$,
or $\Delta_h \uh = \Dm \cdot\sh$,
with~$\sh = \Dp \uh$.
This choice is arbitrary, but it will influence the choice of other
discrete derivative signs.
We will assume the choice of the first option:~$\sh = \Dm \uh$.
Similarly, the direction of the pressure gradient is arbitrary, and we will choose~$\Dp p_h$.
The momentum equation thus reads
\begin{equation}
    - \Dp \cdot \sh + \Dp \ph = \bm{f}.
\end{equation}
We will write the zero divergence condition
in a way that ensures~$\Prodh{\Dp \ph}{\uh} = - \Prodh{\ph}{\Dm \cdot \uh}$
is zero:
\begin{equation}
    \Dm \cdot \uh = 0.
\end{equation}
\begin{lemma}
\label{lem:divh-image}
    Image of the map~$\operatorname{div}_h \colon \Dzh^2 \to \Dh$
    given by~$\operatorname{div}_h \uh = \Dm \cdot \uh$
    is the space~$\Dph$ defined as
    \begin{equation}
        \Dph = \left\{
            f \in \Dh : f|_{\Gamma_p} = 0, \Prodh{f}{1} = 0
        \right\}
    \end{equation}
    where
    \begin{equation}
        \Gamma_p = \left(\{0\} \times [0, 1] \cup [0, 1] \times \{0\}\right) 
        \cap 
        \partial \Omega_h
        \cup \{(1, 1)\}.
    \end{equation}
\end{lemma}

\begin{proof}
    Since~$\Prodh{\Dm \cdot \uh}{\ph} = -\Prodh{\uh}{\Dp \ph}$,
    operator~$\operatorname{grad}_h \colon \Dh \to \Dzh^2$
    given by~$\operatorname{grad}_{0, h} \ph = \pi_0 (\Dp \ph)$
    is the adjoint of~$\operatorname{div}_h$ (up to the sign).
    Therefore, $\operatorname{im} \operatorname{div}_h = \left(\ker \operatorname{grad}_{0,h}\right)^\perp$.
    Suppose~$\operatorname{grad}_{0,h} p = 0$.
    Let~$a = p_{11}$.
    We have~$(\Dp p)_{ij} = 0$ for all~$0 < i, j < N$,
    so~$p_{i+1,j} = p_{ij}$,
    $p_{i,j+1} = p_{ij}$.
    This allows us to conclude that~$p_{ij} = a$ for all the points
    except the corner~$p_{NN}$ and edges~$p_{0j}$, $p_{i0}$.
    In other words, $\operatorname{grad}_{0,h} p = 0$ implies~$p$ is constant
    on~$\Omega_h \setminus \Gamma_p$.
    Conversely, all such functions have zero~$\operatorname{grad}_h$.
    Therefore, the kernel of~$\operatorname{grad}_h$ consists of
    all functions that are constant on~$\Omega_h \setminus \Gamma_p$.
    Its orthogonal complement is comprised of functions that vanish on~$\Gamma_p$,
    and have zero mean over~$\Omega_h \setminus \Gamma_p$,
    which is exactly~$\Dph$.
    \qed
\end{proof}

This lemma suggest the appropriate space for the pressure is~$P_h = \Dph$.
Finally, no additional conditions are needed for~$\sh$,
so we can take~$S_h = \Dh^4$.
The full system has the form
\begin{equation}
\left\{
\begin{aligned}
    - \Dp \cdot \sh + \Dp \ph &= \bm{f} \\
    \Dm \cdot \uh &= 0 \\
    \sh - \Dm \uh &= 0 \\
\end{aligned}
\right.
\end{equation}
Let~$a \colon \Uh \times \Uh \to \mathbb{R}$ be the bilinear form given by
\begin{equation}
\begin{aligned}
    a(u_h, v_h) &= \Prodh{- \Dp \cdot \sh + \Dp \ph}{\vh} \\
    &+ \Prodh{\Dm \cdot \uh}{\qh} \\
    &+ \Prodh{\sh - \Dm \uh}{\th}
\end{aligned}
\end{equation}
This form naturally defines~$A \colon \Uh \to \Uh'$
by~$\langle A u_h, v_h\rangle = a(u_h, v_h)$.
Using formulas for discrete integration by parts,
we can compute its adjoint~$A'$:
\begin{equation}
\begin{aligned}
    \langle A'v_h, u_h\rangle &= \Prodh{\Dp \cdot \th - \Dp \qh}{\uh} \\
    &+ \Prodh{-\Dm \cdot \vh}{\ph} \\
    &+ \Prodh{\th + \Dm \vh}{\sh}
\end{aligned}
\end{equation}
Let~$\Vh = \Uh$ with a different norm: $\Norm{\Vh}{v_h}^2 = \Normh{v_h}^2 + \Norm{\Uh'}{A'v_h}^2$,
where~$\Norm{\Uh'}{A'v_h} = \sup_{u_h \in \Uh, \Normh{u_h} = 1} \left| \langle A'v_h, u_h \rangle\right|$
is the dual norm.
We will henceforth understand~$A$ and~$A'$ as
mappings~$A \colon \Uh \to \Vh'$ and~$A'\colon \Vh \to \Uh'$, respectively.

\begin{theorem}
\label{thm:stokes-well-posed}
    There exists a constant~$\gamma > 0$,
    independent of~$N$,
    such that
    \begin{equation}
        \sup_{\substack{v_h \in \Vh \\ v_h \neq 0}}
        \frac{a(u_h, v_h)}{\ \ \ \Norm{\Vh}{v_h}}
        = \gamma \Norm{\Uh}{u_h} \quad \forall u_h \in \Uh
    \end{equation}
\end{theorem}

To prove this theorem, we need some auxiliary results.

\begin{lemma}
\label{lem:adjoint-isomorphism}
    Operator~$A'$ is an isomorphism.
\end{lemma}

\begin{proof}
    Since we are dealing with finite-dimensional spaces,
    it is enough to show injectivity.
    Suppose~$A'v_h = 0$ for some~$v_h = (\vh, \th, \qh)$.
    Then~$\th = - \Dm\vh$,
    and applying~$A'v_h$ to~$(\vh, \bm{0}, 0)$ we obtain
    \begin{equation}
    \begin{aligned}
        0 &= \Prodh{-\Dp \cdot \th}{\vh} - \Prodh{\Dp \qh}{\vh} \\
          &= \Prodh{\th}{\Dm \vh} + \Prodh{q_h}{\Dm \cdot \vh} \\
          &= \Prodh{-\Dm \vh}{\Dm \vh} + 0 \\
          &= -\Normdh{\vh}^2
    \end{aligned}
    \end{equation}
    so~$\vh = 0$ and~$\th = -\Dm \vh = 0$.
    Finally, by~\cref{lem:divh-image}, there exists~$\uh \in \Dzh^2$
    such that~$\Dm \cdot \uh = \qh$,
    so~$\Normh{\qh}^2 = \Prodh{\qh}{\qh} = \Prodh{\qh}{\Dm \cdot \uh} = -\Prodh{\Dp \qh}{\uh}$.
    But since~$\th = 0$, we have~$-\Prodh{\Dp \qh}{\uh} = \Prodh{-\Dp \cdot \th -\Dp \qh}{\uh} = 0$
    and so~$\qh = 0$.
    \qed
\end{proof}

\begin{lemma}[Discrete Babuška-Aziz inequality]
\label{lem:babuska-aziz}
There exists a constant~$\beta > 0$,
independent of~$N$,
such that
\begin{equation}
    \sup_{\substack{\uh \in V_h \\ \uh \neq 0}}
    \frac{\Prodh{\Dm \cdot \uh}{\ph}}{\Normdh{\uh}}
    \geq
    \beta\Normh{\ph}
\end{equation}
for every~$\ph \in P_h$.
\end{lemma}

\begin{lemma}
\label{lem:div-bound}
Let~$\qh \in P_h$.
There exists~$\vh \in V_h$
such that
\begin{equation}
    \Dm \cdot \vh = \qh
\end{equation}
and~$\Normdh{\vh} \leq \beta^{-1}\Normh{\qh}$
\end{lemma}

\begin{proof}
    Let~$\widehat{V_h}$ denote the subspace of these~$\uh \in V_h$
    that satisfy~$\Dm \cdot \uh = 0$.
    We can decompose~$V_h$ as~$V_h = \widehat{V_h} \oplus \widehat{V_h}^\perp$,
    where~$^\perp$ denotes orthogonal complement with respect to~$\Proddh{\cdot}{\cdot}$.
    Given~$\uh = \widehat{\uh} + \bm{w}_h$, $\widehat{\uh} \in \widehat{V_h}$,
    $\bm{w}_h \in \widehat{V_h}^\perp$ we have for every~$\ph \in P_h$
    \begin{equation}
        \frac{\left|\Prodh{\Dm \cdot \uh}{\ph}\right|}{\Normdh{\uh}}
        =
        \frac{\left|\Prodh{\Dm \cdot \bm{w}_h}{\ph}\right|}{\sqrt{\Normdh{\widehat{\uh}}^2 + \Normdh{\bm{w}_h}^2}}
        \leq
        \frac{\left|\Prodh{\Dm \cdot \bm{w}_h}{\ph}\right|}{\Normdh{\bm{w}_h}^2}
    \end{equation}
    so the supremum of the above expression over all~$\bm{w}_h \in \widehat{V_h}^\perp$
    is the same as over all~$\vh \in V_h$.
    Therefore, by the Babuška-Aziz inequality (\cref{lem:babuska-aziz}),
    \begin{equation}
    \label{eq:subspace-babuska-azis}
        \sup_{\substack{\uh \in \widehat{V_h}^\perp \\ \uh \neq 0}}
        \frac{\Prodh{\Dm \cdot \uh}{\ph}}{\Normdh{\uh}}
        \geq
        \Normh{\ph}
    \end{equation}
    
    Let~$b \colon \widehat{V_h}^\perp \times P_h \to \mathbb{R}$ be a bilinear form
    given by~$b(\uh, \ph) = \Prodh{\Dm \cdot \uh}{\ph}$.
    We have a pair of adjoint operators:
    \begin{equation}
    \begin{aligned}
        B &\colon \widehat{V_h}^\perp \to P_h', 
        & \langle B\vh, \ph \rangle &= \Prodh{\Dm \cdot \uh}{\ph} \\
        B' &\colon P_h \to (\widehat{V_h}^\perp)', 
        & \langle B'\ph, \vh \rangle &= - \Prodh{\uh}{\Dp \ph}
    \end{aligned}
    \end{equation}
    \Cref{eq:subspace-babuska-azis} is equivalent to the statement
    that~$B'$ satisfies the inf-sup condition
    with constant~$\beta$,
    that is~$\Norm{V_h'}{B'\ph} \geq \beta \Normh{\ph}$.
    Since~$B$ is injective and~$B'$ is bounded below,
    by Theorem 1 of~\cite{demkowicz_babuska_2006},
    $B$ is bounded below with the same constant,
    and by the Babuška Theorem,
    problem~$B\vh =\qh$ has a unique solution satisfying~$\Normdh{\vh} \leq \beta^{-1} \Normh{\qh}$.
    \qed
\end{proof}

\begin{lemma}
\label{lem:adjoin-bound}
    There exists~$\alpha > 0$,
    independent of~$N$,
    such that~$\Norm{\Uh'}{A'v_h} \geq \alpha \Normh{v_h}$
    for all~$v_h \in \Vh$.
\end{lemma}

\begin{proof}
    Let~$\Vh \ni v_h = (\vh, \th, \qh)$
    and~$A'v_h = f$.
    We can decompose~$f \in \Uh'$
    as~$\langle f, u_h\rangle = 
    \Prodh{\bm{f}}{\uh} + 
    \Prodh{\bm{h}}{\sh} +
    \Prodh{g}{\ph}$
    for~$\bm{f} \in V_h$, $\bm{h} \in S_h$ and~$g \in P_h$.
    Since~$A'$ is an isomorphism by~\cref{lem:adjoint-isomorphism},
    there exist~$v_h^{\bm{f}}, v_h^{\bm{h}}, v_h^g \in \Vh$
    such that~$\langle Av_h^{\bm{f}}, u_h\rangle = \Prodh{\bm{f}}{\uh}$,
    $\langle Av_h^{\bm{h}}, u_h\rangle = \Prodh{\bm{h}}{\sh}$,
    and~$\langle Av_h^g, u_h\rangle = \Prodh{g}{\ph}$.

    \begin{claim}
        $\Normh{v_h^{\bm{f}}} \leq C_1\Normh{\bm{f}}$
        for some constant~$C_1 > 0$,
        independent of~$N$.
    \end{claim}
    
    \noindent
    Let~$v_h^{\bm{f}} = \left(\vhf, \thf, \qhf\right)$.
    We have~$\thf + \Dm \cdot \vhf = 0$,
    so~$\Dp \cdot \thf = -\Delta_h \vhf$
    and thus
    \begin{equation}
    \label{eq:vhf-momentum}
        \Prodh{-\Delta_h \vhf - \Dp \qhf}{\uh} = \Prodh{\bm{f}}{\uh}
    \end{equation}
    for all~$\uh \in V_h$.
    Putting~$\uh = \vhf$ in the above equation we get
    \begin{equation}
        \Prodh{-\Delta_h \vhf}{\vhf} - \Prodh{\Dp \qhf}{\vhf} = \Prodh{\bm{f}}{\vhf}.
    \end{equation}
    Integrating the first term by parts
    and noting that
    \begin{equation}
    - \Prodh{\Dp \qhf}{\vhf} = \Prodh{\qhf}{\Dm \cdot \vhf} = 0    
    \end{equation}
    we get
    \begin{equation}
        \Normdh{\vhf}^2 = \Normh{\Dp\vhf}^2 = \Prodh{\bm{f}}{\vhf}
        \leq
        \Normh{\bm{f}} \Normh{\vhf}.
    \end{equation}
    Let~$K$ be the discrete Poincaré constant from Lemma \ref{thm:discrete-poincare}.
    Then
    \begin{equation}
        \Normdh{\vhf}^2
        \leq
        \Normh{\bm{f}} \Normh{\vhf}
        \leq
        K \Normh{\bm{f}} \Normdh{\vhf}
    \end{equation}
    so~$\Normdh{\vhf} \leq K \Norm{V_h'}{\bm{f}}$.
    Since~$\thf = - \Dm \cdot \vhf$,
    we have~$\Normh{\thf} = \Normdh{\vhf} \leq K \Normh{\bm{f}}$.
    Furthermore, $\Normh{\vhf} \leq K \Normdh{\vhf} \leq K^2\Normh{\bm{f}}$.
    As for~$\qhf$, from~\cref{eq:vhf-momentum} for every~$\uh \in V_h$ we have
    \begin{equation}
        -\Prodh{\Dp \qhf}{\uh} = \Prodh{\Delta_h\vhf}{\uh} + \Prodh{\bm{f}}{\uh}.
    \end{equation}
    Integrating by parts we get
    \begin{equation}
    \begin{aligned}
        \Prodh{\qhf}{\Dm \cdot \uh} &= - \Prodh{\Dp\vhf}{\Dp\uh} + \Prodh{\bm{f}}{\uh} \\
        &\leq \Normdh{\vhf} \Normdh{\uh} + \Normh{\bm{f}} \Normh{\uh} \\
        &\leq 2K \Normh{\bm{f}} \Normdh{\uh}.
    \end{aligned}
    \end{equation}
    Using the discrete Babuška-Aziz inequality (\cref{lem:babuska-aziz})
    we get
    \begin{equation}
        \Normh{\qhf} \leq \beta^{-1}\sup_{\substack{\uh \in V_h \\ \uh \neq 0}}
        \frac{\Prodh{\qhf}{\Dm \cdot \uh}}{\Normdh{\uh}}
        \leq
        2K\beta^{-1} \Normh{\bm{f}}.
    \end{equation}
    Putting together upper bounds of~$\vhf$, $\thf$ and~$\qhf$
    provides a bound of~$v_h^f$ and proves the claim.
    \qed

    \begin{claim}
        $\Normh{v_h^{\bm{h}}} \leq C_2\Norm{s_h'}{\bm{h}}$
        for some constant~$C_2 > 0$,
        independent of~$N$.
    \end{claim}

    \noindent
    We have~$\thh = -\Dm\vhh + \bm{h}$, so
    \begin{equation}
        \Prodh{-\Delta_h \vhh}{\uh} - \Prodh{\Dp\qhh}{\uh} = \Prodh{\Dp \cdot \bm{h}}{\uh}.
    \end{equation}
    Letting~$\uh = \vhh$ and integrating by parts, we obtain
    \begin{equation}
        \Prodh{\Dp\vhh}{\Dp\vhh} + \Prodh{\qhh}{\Dm \cdot \vhh} = -\Prodh{\bm{h}}{\Dm \vhh}.
    \end{equation}
    Since~$\Prodh{\qhh}{\Dm \cdot \vhh} = 0$, we are left with
    \begin{equation}
        \Normdh{\vhh}^2 = -\Prodh{\bm{h}}{\Dm \vhh} \leq \Normh{\bm{h}} \Normdh{\vhh},
    \end{equation}
    so~$\Normdh{\vhh} \leq \Normh{\bm{h}}$.
    This implies~$\Normh{\vhh} \leq K\Normh{\bm{h}}$ and
    \begin{equation}
        \Normh{\thh} = \Normh{\Dm\vhh + \bm{h}} \leq \Normh{\Dm\vhh} + \Normh{\bm{h}}
        \leq (1 + K)\Normh{\bm{h}}.
    \end{equation}
    Bound for~$\qhh$ can be establish in the same way as the one for~$\qhf$.
    For every~$\uh \in V_h$ we have
    \begin{equation}
        -\Prodh{\Dp \qhh}{\uh} = \Prodh{\Delta_h\vhh}{\uh} + \Prodh{\Dp \cdot \bm{h}}{\uh}.
    \end{equation}
    so
    \begin{equation}
    \begin{aligned}
        \Prodh{\qhh}{\Dm \cdot \uh} &= - \Prodh{\Dp\vhh}{\Dp\uh} - \Prodh{\bm{h}}{\Dm\uh} \\
        &\leq \Normdh{\vhh} \Normdh{\uh} + \Normh{\bm{h}} \Normdh{\uh} \\
        &\leq 2 \Normh{\bm{h}} \Normdh{\uh}.
    \end{aligned}
    \end{equation}
    and thus by~\cref{lem:babuska-aziz}
    \begin{equation}
        \Normh{\qhh} \leq \beta^{-1}\sup_{\substack{\uh \in V_h \\ \uh \neq 0}}
        \frac{\Prodh{\qhh}{\Dm \cdot \uh}}{\Normdh{\uh}}
        \leq
        2\beta^{-1} \Normh{\bm{h}},
    \end{equation}
    which is the desired bound.
    \qed

    \begin{claim}
        $\Normh{v_h^g} \leq C_3\Normh{g}$
        for some constant~$C_3 > 0$,
        independent of~$N$.
    \end{claim}

    \noindent
    By~\cref{lem:div-bound}, there exists~$\vh^0 \in V_h$
    such that
    \begin{equation}
        -\Dm \cdot \vh^0 = g, \quad
        \Normdh{\vh^0} \leq \beta^{-1}\Normh{g}.
    \end{equation}
    Let~$\wh = \vhg - \vh^0$,
    then~$\Dm \cdot \wh = 0$ and~$\thg + \Dm \wh = -\Dm \vh^0$,
    so~$\widetilde{w}_h = (\wh, \thg, \qhg)$
    satisfies
    \begin{equation}
    \langle A\widetilde{w}_h, u_h\rangle = - \Prodh{\Dm\vh^0}{\sh}
    \end{equation}
    for all~$u_h \in \Uh$.
    The previous claim implies that~$\Normh{\widetilde{w}_h} \leq C\Normh{\Dm\vh^0}$.
    Given that~$\Normh{\Dm\vh^0} \leq \beta^{-1}\Normh{g}$,
    we have
    \begin{equation}
    \begin{aligned}
        \Normh{v_h^g} &= \Normh{\widetilde{w}_h + (\vh^0, \bm{0}, 0)} \\
        &\leq \Normh{\widetilde{w}_h} + \Normh{\vh^0} \\
        &\leq C\beta^{-1}\Normh{g} + K\Normdh{\vh^0} \\
        &\leq \beta^{-1}(C + K)\Normh{g},
    \end{aligned}
    \end{equation}
    which bounds~$\Normh{v_h^g}$ by~$\Normh{g}$ as required.
    \qed

    \medskip
    Now we are ready to prove the main statement of the lemma.
    By linearity, $v_h = v_h^{\bm{f}} + v_h^{\bm{h}} + v_h^g$,
    so
    \begin{equation}
    \begin{aligned}
        \Normh{v_h}^2
        &= \Normh{v_h^{\bm{f}} + v_h^{\bm{h}} + v_h^g}^2 \\
        &= 3\left(
            \Normh{v_h^{\bm{f}}}^2 +
            \Normh{v_h^{\bm{h}}}^2 +
            \Normh{v_h^g}^2
        \right) \\
        &\leq 3\left(
            C_1^2\Normh{\bm{f}}^2 +
            C_2^2\Normh{\bm{h}}^2 +
            C_3^2\Normh{g}^2
        \right) \\
        &\leq \underbrace{3 \max\{C_1,C_2,C_3\}^2}_{\alpha^{-2}}\left(
            \Normh{\bm{f}}^2 +
            \Normh{\bm{h}}^2 +
            \Normh{g}^2
        \right) \\
        &= \alpha^{-2} \Norm{\Uh'}{f}^2 \\
        &= \alpha^{-2} \Norm{\Uh}{A'v_h}^2
    \end{aligned}
    \end{equation}
    since~$\bm{f}$, $\bm{g}$, $h$ extended to~$\Uh'$ are orthogonal,
    and thus~$\Norm{\Uh}{A'v_h} \geq \alpha \Normh{v_h}$.
\end{proof}

\begin{proof}[of~\cref{thm:stokes-well-posed}]
By~\cref{lem:adjoin-bound} we have
\begin{equation}
    \Norm{\Vh}{v_h}^2 = \Normh{v_h}^2 + \Norm{\Uh'}{A'v_h}^2
    \leq
    \left(\alpha^{-2} + 1\right)\Norm{\Uh'}{A'v_h}^2,
\end{equation}
so
\begin{equation}
    \Norm{\Uh'}{A'v_h} \geq \frac{\alpha}{\sqrt{1 + \alpha^2}}\Norm{\Vh}{v_h}.
\end{equation}
This means~$A'$ is bounded below by~$\gamma = \alpha / \sqrt{1 + \alpha^2}$.
Since~$A'$ is an isomorphism (\cref{lem:adjoint-isomorphism}),
the same is true about~$A$.
In particular, $A$ is injective.
By Theorem~1 of~\cite{demkowicz_babuska_2006},
$A$ is bounded below by~$\gamma$.
\qed
\end{proof}

\begin{figure}
    \centering
    \includegraphics[width=0.7\linewidth]{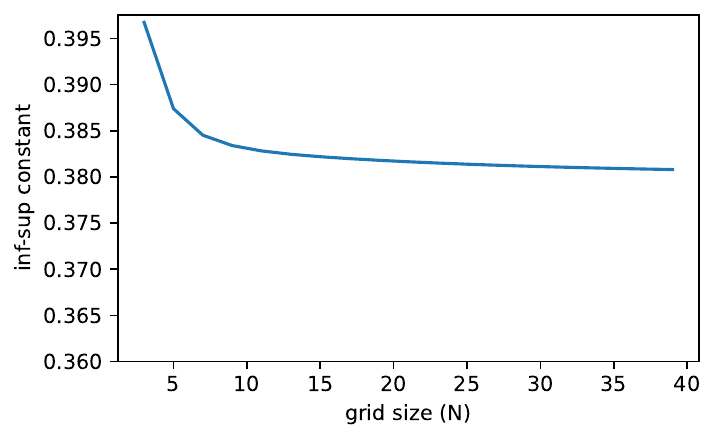}
    \caption{%
    Dependence of the constant from Lemma \ref{lem:babuska-aziz} on the mesh size.}
    \label{fig:babuska-aziz}
\end{figure}

\section{Programming environment}
\label{sec:programming}

We provide a Python environment for defining discrete weak formulations.
Our Python library 
\textit{DVF-CRVPINN}\footnote{\url{https://github.com/marcinlos/dvf}}, provides tools for prototyping Discrete Variational Formulations.
It employs standard scientific computing libraries, such as \textit{NumPy} and \textit{SciPy}.  \textit{matplotlib}, 
as well as \textit{PyTorch} for neural network training.
This environment can be run locally on a laptop or workstation equipped with Google Colab\footnote{\url{https://colab.research.google.com/drive/1DWFaix6nFeod6DExAQ8Ka8bEdQdFDKP2}}.
The goal of this library is to allow the building of complex formulations easily. 

We will explain now how to use this library on the simple example of Laplace problem.
We start from defining grid of $31\times 31$ collocation points,  representing the discrete domain~$\Omega_h$.

\begin{lstlisting}
grid = Grid(30, 30)
assert grid.shape == (31, 31)
assert grid.size = 31 ** 2
X, Y = grid.points  # array of shape (2, 31, 31)
assert grid.cell_volume == 1 / 30**2
\end{lstlisting}

The \lstinline{Grid} class allows simple iteration over points, indices, boundary points, etc. It also allows for the management of boundary normal vectors and the computation of cell/face volumes.

The next step is to define a functionality of a \lstinline{GridFunction} on a \lstinline{Grid} of points. 
It supports scalar, vector, and tensor values of discrete functions. Our library supports arithmetic operations of discrete functions
    \lstinline{f + g}, \lstinline{f * g}, \lstinline{f / g}, etc.;
 it allows concatenating these discrete functions with other Python functions such as     \lstinline{sin(f)}, \lstinline{exp(f)}, \lstinline{sqrt(f)}, etc.
It also supports  differential operators including gradient
    \lstinline{grad(f, "+")}, divergence \lstinline{div(F, "-")}, etc.
It is also possible to convert the discrete function to or from  \lstinline{numpy} arrays.

\begin{lstlisting}
f = GridFunction.from_function(lambda x, y: x + y, grid)
assert f(i, j) == grid.points[0, i, j] + grid.points[1, i, j]
vals = f.tabulate()     # vals.shape == (N, N)
g = GridFunction.from_array(vals, grid)
assert f == g
df = grad(f, "+")
dvals = df.tabulate()   # dvals.shape == (2, N, N)
\end{lstlisting}

Using the \textit{SymPy} library enables the creation of manufactured solution problems. We illustrate this using the example of $u_\text{exact}=\sin(\pi x) \sin(3\pi y)$ defined over the discrete domain $\Omega_h =\{ (i / 30,j / 30): 0\leq i,j \leq 30\} \subset \Omega=[0,1]^2$.

\begin{lstlisting}
from sympy import diff, sin, lambdify, pi

def laplacian(expr, *vars):
    return sum(diff(expr, v, v) for v in vars)

x, y = sympy.symbols("x y")
u_exact = sin(pi * x) * sympy.sin(3 * pi * y)
rhs = -laplacian(u_exact, x, y)

u_grid = GridFunction.from_function(lambdify([x, y], u_exact), grid)
rhs_grid = GridFunction.from_function(lambdify([x, y], rhs), grid)
\end{lstlisting}

Discrete weak formulations requires tools for computing discrete integrals, defined as appropriate sums over the values defined at the grid points, or the sums of the discrete finite-difference derivatives. We provide tools to define discrete  $\Normh{\cdot}$, $\Normdh{\cdot}$ norms, tools for defining discrete 
$\Prodh{\cdot}{\cdot}$, $\Proddh{\cdot}{\cdot}$ scalar products, as well as tools to compute discrete  integrals of arbitrary \lstinline{GridFunctions} over \lstinline{Grid} points, and discrete boundary  integrals computed over \lstinline{Grid} points and boundary points,
\begin{lstlisting}
fg = integrate(f * g)
grad_prod = product(grad(f, "+"), grad(g, "+"), "h")
grad_prod2 = product(f, g, "grad_h")
bd_term = integrate_bd(f)

n_x = GridFunction(lambda i, j: grid.facet_normal(i, j)[0], grid)
a = integrate(Dx(f, "+") * g)
b = -integrate(f * Dx(g, "-")) + integrate_bd(f * g * n_x)
assert a == b
\end{lstlisting}
%
The library supports scalar, vector, and arbitrary tensor function spaces,
\begin{lstlisting}
P = FunctionSpace(grid)                 # scalars
U = VectorFunctionSpace(grid, 2)        # 2 element vectors
S = TensorFunctionSpace(grid, (2, 2))   # 2 x 2 tensors
\end{lstlisting}
It also supports products of spaces,
\begin{lstlisting}
W = CompositeFunctionSpace(S, U, P) 
assert U.dim == 2 * P.dim
assert S.shape == (2, 2)
assert U.zero == np.array([0, 0])
# indices of basis functions on the x = 0 boundary
select_dofs(P, lambda i, j: i == 0)  
\end{lstlisting}

The next step is to use our library to construct discrete weak formulations.
We provide an easy and intuitive interface for the creation of bilinear and linear forms
inspired by FEniCS and the Unified Form Language (UFL),
\begin{lstlisting}
u = U.trial_function()
v = U.test_function()

def B_form(u, v):
    return dot(grad(u, "+"), grad(v, "+"))
\end{lstlisting}
We also directly support composite spaces
\begin{lstlisting}
uu = W.trial_function()
vv = W.test_function()

sigma, u, p = uu.components
tau, v, q = vv.components
\end{lstlisting}

\subsection{Formulation of the discrete weak Laplace problem}

For the Laplace problem:
\begin{equation}
\left\{
\begin{aligned}
    -\Delta u &= f \\
    \left.u\right|_{\partial\Omega} &= 0
\end{aligned}
\right.
\end{equation}
we will use the bilinear form~$b(u, v) = \Prodh{\Dxp u}{\Dxp v}$ and the~$\Normdh{\cdot}$ test space norm
(see~\cite{los_collocation-based_2025} for derivation and proofs of relevant bounds).

We start by defining the grid and scalar function space over it.
\begin{lstlisting}
grid = Grid(30, 30)
U = FunctionSpace(grid)
\end{lstlisting}
To apply the Dirichlet boundary conditions on the entire boundary,
we first define a function that selects all the boundary points.
\begin{lstlisting}
def U_mask_fun(i, j):
    edges_x = (0, grid.n[0])
    edges_y = (0, grid.n[1])
    m = 0 if i in edges_x or j in edges_y else 1
    return m
\end{lstlisting}
We can use it to select the corresponding degrees of freedom from the~\lstinline{U} space.
\begin{lstlisting}
u_mask = GridFunction(U_mask_fun, grid)
U_bc = select_dofs(U, u_mask, invert=True)
\end{lstlisting}
The definition of the bilinear form and assembling its matrix is now simple.
\begin{lstlisting}
def B_form(u, v):
    return dot(grad(u, "+"), grad(v, "+"))

u = U.trial_function()
v = U.test_function()

B = np.zeros((U.dim, U.dim))
assemble(B_form(u, v), B, u, v)
B_ = remove_dofs(B, U_bc)
\end{lstlisting}
We can define the right-hand side vector using the \lstinline{rhs} function
defined in the \emph{SymPy} example from section~\cref{sec:programming}.
\begin{lstlisting}
def vector_of_values(*funs):
    return np.concat([np.ravel(f.tabulate()) for f in funs])

rhs_f = GridFunction.from_function(sympy.lambdify([x, y], rhs), grid)
rhs_vec = remove_dofs(vector_of_values(rhs_f), U_bc)
\end{lstlisting}

The robust discrete variational physics-informed neural network will be trained based on the residual loss multiplied by the inverse of the Gram matrix computed in the proper discrete product,
namely
\begin{equation}
\begin{aligned}
    \mathcal{L}(\theta) &= \Norm{V'}{Au_\theta - f}^2  &=R_\theta^T G R_\theta 
    &=\text{RES}\,(\theta)^T G^{-1}\, \text{RES}(\theta)
\end{aligned}
\end{equation}
where
$R_\theta = G^{-1}\,\text{RES}(\theta)$ is the vector representing $v \mapsto \Prodh{Au - f}{v}$,
$\text{RES}(\theta)_{i} = \Prodh{Au - f}{v_i}$ and $v_i$ is the~$i$-th basis function of~$V$.
We compute the LU factorization of the Gram matrix of our~$\Proddh{\cdot}{\cdot}$ scalar product,
which coincides with the matrix of our bilinear form.
\begin{lstlisting}
G_ = B_
G_LU = scipy.linalg.lu_factor(G_)
\end{lstlisting}
Since the loss function is evaluated many times, each loss computation reuses the factorization
\begin{lstlisting}
residuum_vec = B_ @ vals_vec - M_ @ rhs_vec
residuum_rep = scipy.linalg.lu_solve(G_LU, residuum_vec)
loss = np.dot(residuum_vec, residuum_rep)
\end{lstlisting}
For efficient training of the neural network using the robust loss, to enable backpropagation when using 
\lstinline{scipy.linalg.lu_solve},
we can wrap the loss computation in a~\lstinline{torch.autograd.Function}:
\begin{lstlisting}
class ResiduumNormSq(torch.autograd.Function):
    @staticmethod
    def forward(vals):
        vals_vec = remove_dofs(vals.cpu().numpy(), W_bc)
        residuum_vec = B_ @ vals_vec - M_ @ rhs_vec
        residuum_rep = scipy.linalg.lu_solve(G_LU, residuum_vec)
        return torch.tensor(np.dot(residuum_vec, residuum_rep))
    ...
\end{lstlisting}
We also need to compute the derivative with respect to the input vector
\begin{equation}
    \mathcal{L}(\bm{v}) = (A\bm{v} - \bm{f})^TG^{-1}(A\bm{v} - \bm{f})
    \implies
    \frac{\partial \mathcal{L}}{\partial \bm{v}} = 2 A^TG^{-1}(A\bm{v} - \bm{f})
\end{equation}
This functionality is implemented as follows :
\begin{lstlisting}
    @staticmethod
    def setup_context(ctx, args, output):
        (vals,) = args
        ctx.save_for_backward(vals)

    @staticmethod
    def backward(ctx, grad_output):
        (vals,) = ctx.saved_tensors
        vals_vec = remove_dofs(vals.cpu().numpy(), W_bc)
        residuum_vec = B_ @ vals_vec - M_ @ rhs_vec
        residuum_rep = scipy.linalg.lu_solve(G2_LU, residuum_vec)
        full_residuum_rep = reinsert_dofs(B_.T @ residuum_rep, W_bc)
        grad_vals = grad_output * 2 * torch.tensor(full_residuum_rep)
        return grad_vals.to(device)
\end{lstlisting}
With these tools in hand, we are ready to define the loss function used for training:
\begin{lstlisting}
def loss_function(pinn):
    points = np.concat(grid.points).reshape(2, -1)
    args = torch.from_numpy(points.T.astype(np.float32)).to(device)
    values = torch.ravel(pinn(args).t())
    return ResiduumNormSq.apply(values)
\end{lstlisting}
which can then be used with standard~\emph{PyTorch} optimizers. 

\subsection{Formulation of the discrete weak Stokes problem}

In this section, we will show, step by step, how to implement the Stokes problem
formulation from section~\cref{sec:stokes-formulation}.
The \lstinline{grid} object remains the same as in the Laplace problem,
but the spaces are more complex:
\begin{lstlisting}
S = TensorFunctionSpace(grid, (2, 2))
U = VectorFunctionSpace(grid, 2)
P = FunctionSpace(grid)
W = CompositeFunctionSpace(S, U, P)
\end{lstlisting}
The bilinear form~$A$:
\begin{equation}
\begin{aligned}
    A(u, v) &= 
    \Prodh{-\Dp\cdot\bm{\sigma} + \Dp p}{\bm{v}} \\
    & + \Prodh{\Dm \cdot \bm{u}}{q} \\
    &+ \Prodh{\bm{\sigma} - \Dm \bm{u}}{\bm{\tau}}
\end{aligned}
\end{equation}
becomes
\begin{lstlisting}
def A_form(sigma, u, p, tau, v, q):
    return (
        dot(-div(sigma, "+") + grad(p, "+"), v)
        + dot(div(u, "-"), q)
        + dot(sigma - grad(u, "-"), tau)
    )
\end{lstlisting}
The $\Prodh{\cdot}{\cdot}$ scalar product:
\begin{equation}
    \Prodh{u}{v} = \Prodh{\bm{\sigma}}{\bm{\tau}}
    + \Prodh{\bm{u}}{\bm{v}}
    + \Prodh{p}{q}
\end{equation}
is implemented as
\begin{lstlisting}
def L2_product(sigma, u, p, tau, v, q):
    return dot(sigma, tau) + dot(u, v) + p * q
\end{lstlisting}
Finally, the~$\Prod{A^*}{\cdot}{\cdot}$ test space scalar product given by
\begin{equation}
\begin{aligned}
    \Prod{A^*}{u}{v} &= 
    \Prodh{\pi_0(\Dp \cdot \bm{\sigma} - \Dp p)}
          {\pi_0(\Dp \cdot \bm{\tau} - \Dp q)} \\
    &+ \Prodh{\Dm \cdot \bm{u}}{\Dm \cdot \bm{v}} \\
    &+ \Prodh{\bm{\sigma} + \Dm \bm{u}}{\bm{\tau} + \Dm \bm{v}}
\end{aligned}
\end{equation}
can be expressed as
\begin{lstlisting}
def AT_product(sigma, u, p, tau, v, q):
    return (
        dot(pi0(div(sigma, "+") - grad(p, "+")), 
            pi0(div(tau, "+") - grad(q, "+")))
        + dot(div(u, "-"), div(v, "-"))
        + dot(sigma + grad(u, "-"), tau + grad(v, "-"))
    )
\end{lstlisting}
The discrete weak formulation is then assembled into matrices
\begin{lstlisting}
assemble(A_form(sigma, u, p, tau, v, q), A, uu, vv)
assemble(L2_product(sigma, u, p, tau, v, q), M, uu, vv)
\end{lstlisting}
and the right-hand side vector is
\begin{lstlisting}
def vector_of_values(*funs):
    return np.concat([np.ravel(f.tabulate()) for f in funs])

rhs_vec = vector_of_values(S.zero_fun, rhs_f, P.zero_fun)
\end{lstlisting}
We also provide functionality to select boundary points for Dirichlet conditions
\begin{lstlisting}
def U_mask_fun(i, j):
    edges_x = (0, grid.n[0])
    edges_y = (0, grid.n[1])
    return 0 if i in edges_x or j in edges_y else 1

u_mask = GridFunction(U_mask_fun, grid)
U_bc = select_dofs(U, u_mask, invert=True)
\end{lstlisting}
For composite spaces, we can combine subspaces of degrees of freedom sets
\begin{lstlisting}
S_bc = select_dofs(S, s_mask, invert=True)
U_bc = select_dofs(U, u_mask, invert=True)
P_bc = select_dofs(P, p_mask, invert=True)

W_bc = W.combine_dofs(S_bc, U_bc, P_bc)
\end{lstlisting}
Degrees of freedom corresponding to Dirichlet boundary conditions are
physically removed from the arrays
\begin{lstlisting}
A_ = remove_dofs(A, W_bc)
M_ = remove_dofs(M, W_bc)
rhs_vec = remove_dofs(rhs_vec, W_bc)
\end{lstlisting}
Once the matrix and the right hand side are assembled,
the training process is the same as for the Laplace problem example.

\section{Numerical results}

In this section, we present numerical experiments illustrating the accuracy
and convergence history of the training for two versions of the Stokes problem: 
one with a manufactured solution and the cavity flow problem.
In both experiments, we use a neural network with two hidden layers, $512$ neurons each,
and train them with the Adamax optimizer (\lstinline{torch.optim.Adamax}) with default settings.
The results are obtained using PyTorch 2.9.1.

\subsection{Stokes problem with manufactured solution}

In the first numerical test, we use the manufactured solution problem,
where the exact solution is given by
\begin{equation}
\begin{aligned}
    \bm{u}(x, y) =&
    \begin{bmatrix}
     2e^x(x - 1)^2 x^2 y(y - 1)(2y - 1) \\
     -e^x (x - 1)x(x^2 + 3x - 2)(y - 1)^2 y^2
    \end{bmatrix} \\
    p(x, y) =& -424 + 156e
            + y(y - 1)(-456 + e^x
                (
                    456
                    + x^2 (228 - 5y^2 + 5y)\\&
                    + 2x (y^2 - y -228)
                    + 2x^3 (y^2 - y - 36)
                    + x^4 (y^2 - y + 12)
                )
            )
\end{aligned}
\end{equation}
We train the PINN for 5000 epochs using a $20 \times 20$ and $40\times 40$ grid sizes.
The final results are presented in~\cref{fig:pinn20x20} and~\cref{fig:pinn40x40},
respectively.
While the $20 \times 20$ solution exhibits some distortions and asymmetry,
the results obtained on the $40 \times 40$ grid are close to the exact analytical solution
given above.
The loss and true error throughout the training process are presented in~\cref{fig:training-manufactured}.
It clearly shows that the error does stay within the bounds predicted by the theoretical analysis
carried out in~\cref{sec:stokes-formulation}.
In fact, the error is rather close to the lower bound.
This can be explained by the fact that the upper bound is controlled
by the smallest generalized eigenvalues of the~$\bm{A}^T\bm{G}^{-1}\bm{A}\bm{x} = \lambda \bm{G}\bm{x}$
system, while most of these eigenvalues are significantly larger
(see~\cite{los_collocation-based_2025} for more details).

\begin{figure}
\begin{center}
    \includegraphics[width=\textwidth]{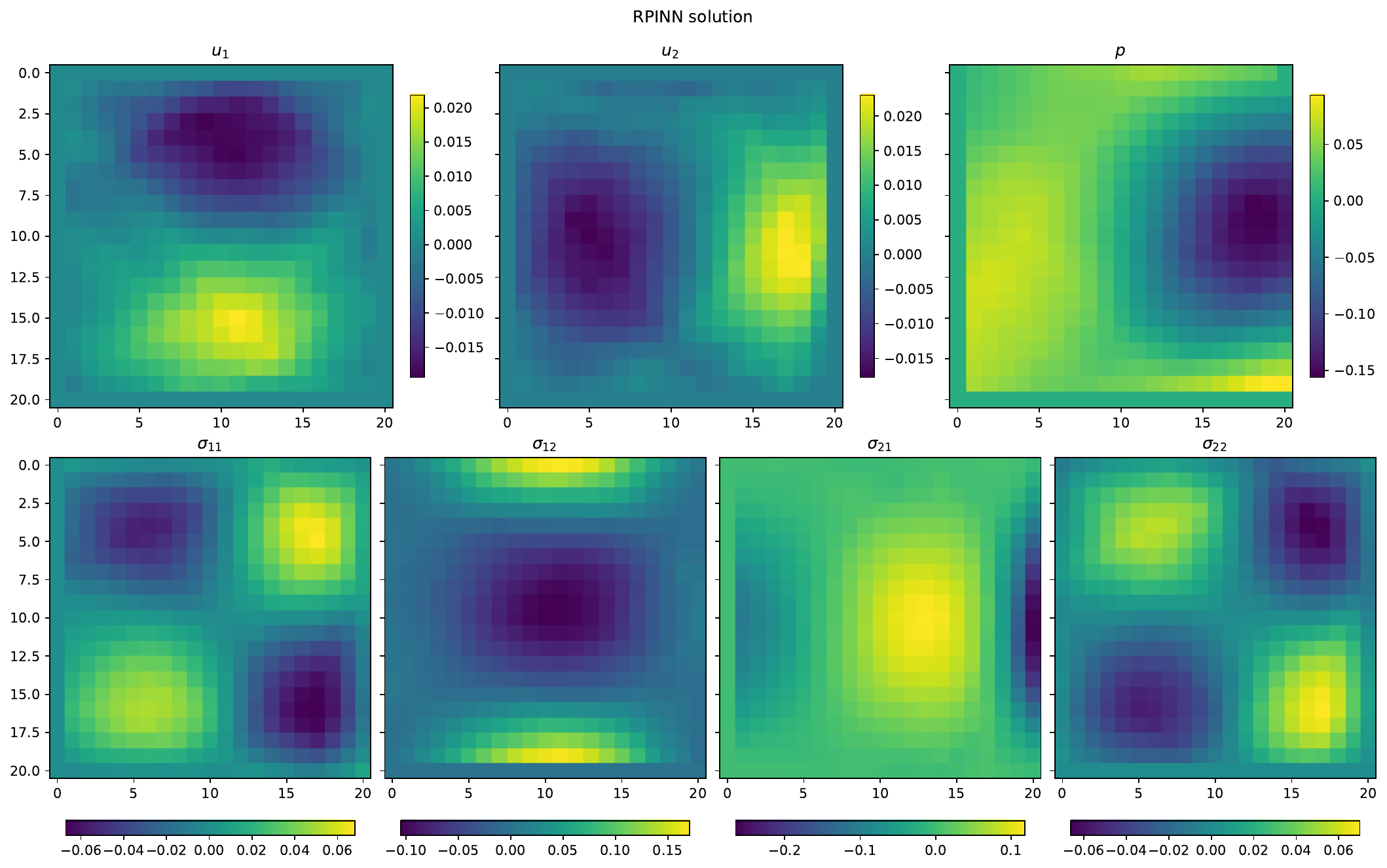}
\end{center}
    \caption{Manufactured solution Stokes problem, $20 \times 20$ grid.}
    \label{fig:pinn20x20}
\end{figure}

\begin{figure}
\begin{center}
    \includegraphics[width=\textwidth]{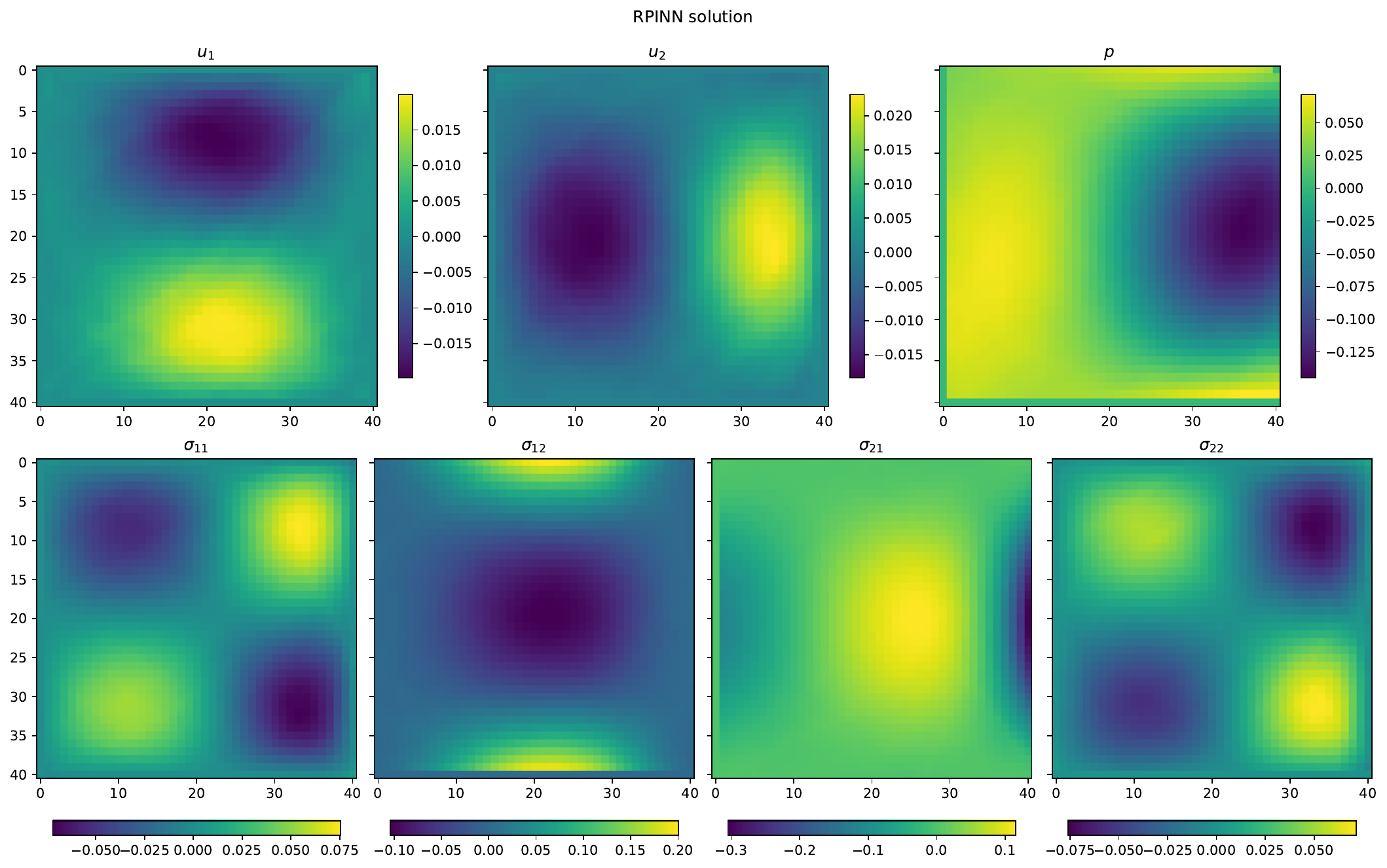}
\end{center}
    \caption{Manufactured solution Stokes problem, $40 \times 40$ grid.}
    \label{fig:pinn40x40}
\end{figure}

\begin{figure}
    \centering
    \includegraphics[width=0.48\linewidth]{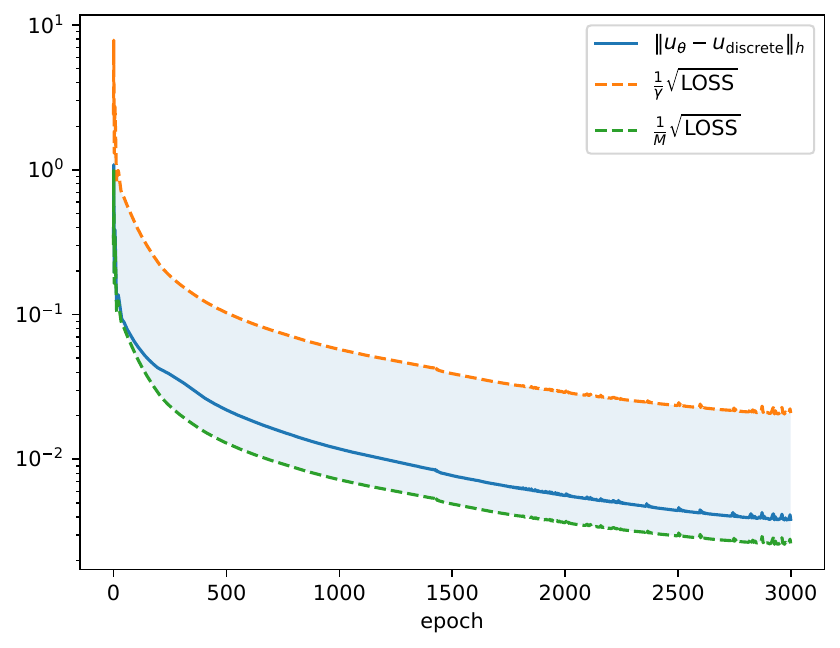}
    ~
    \includegraphics[width=0.48\linewidth]{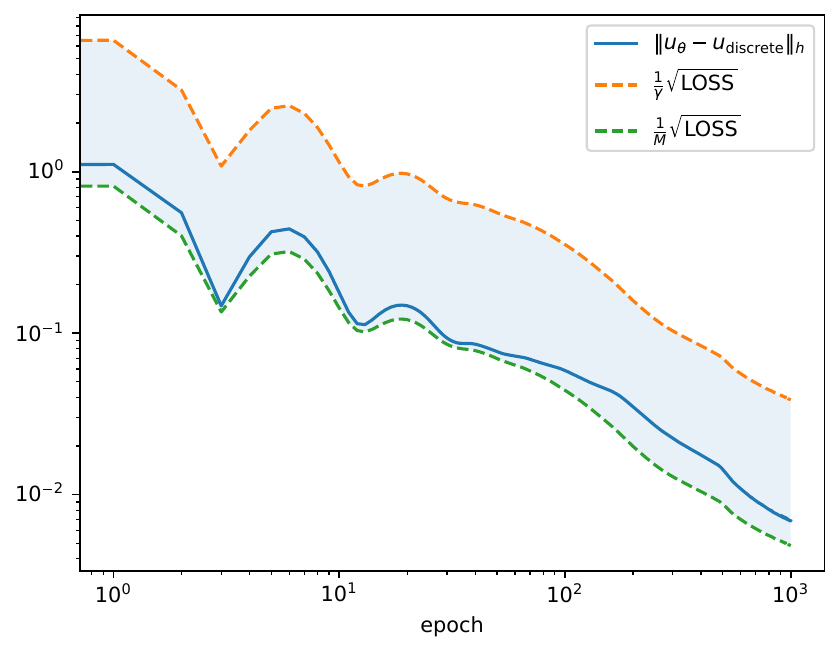}
    \caption{%
    Loss function value and true error throughout the training process
    for the manufactured solution Stokes problem.
    The right figure presents the same data as the left one in logarithmic scale,
    to better show that the lower and upper bounds hold in the initial phase
    of the training process.}
    \label{fig:training-manufactured}
\end{figure}

\subsection{Cavity flow problem}

The second test involves the two-dimensional lid cavity flow problem.
Here the right-hand side~$\bm{f}$ is zero,
and we use non-homogeneous Dirichlet boundary condition on velocity:
$\left.\bm{u}\right|_{\partial\Omega} = \bm{g}$,
where
\begin{equation}
    \bm{g}_1(x, y) =
    \begin{cases}
    1 & \text{if }y = 1\\
    0 & \text{otherwise}
    \end{cases},
    \quad
    \bm{g}_2(x, y) = 0
\end{equation}
that is, horizontal velocity is~$1$ on the upper boundary.
We can define the boundary condition values as follows.
\begin{lstlisting}
def u1_bc(x, y):
    return np.isclose(y, 1)

def u2_bc(x, y):
    return np.zeros_like(x)
\end{lstlisting}
The use of non-homogeneous Dirichlet boundary conditions requires a slight change
in the right-hand side assembly to account for the shift function:
\begin{lstlisting}
X, Y = grid.points
u_bc_data = np.array([u1_bc(X, Y), u2_bc(X, Y)])
u_bc_fun = GridFunction.from_array(u_bc_data, grid)

rhs_forcing = M @ vector_of_values(S.zero_fun, rhs_f, P.zero_fun)
rhs_bc = A @ vector_of_values(S.zero_fun, u_bc_fun, P.zero_fun)
rhs_vec = remove_dofs(rhs_forcing - rhs_bc, W_bc)
\end{lstlisting}

The PINN solution obtained after 25,000 epochs of training is presented in~\cref{fig:trained-network}.
It is fairly close to the exact solution and does not exhibit visible irregularities.
The loss and true error throughout the training process are presented in~\cref{fig:training-min}.
Since we have no analytic formulas for the exact solution of the cavity flow problem,
the error is measured with respect to a reference solution.
Like in the manufactured solution problem, the error obeys the theoretical lower and upper bounds.

\begin{figure}
    \centering
    \includegraphics[width=\linewidth]{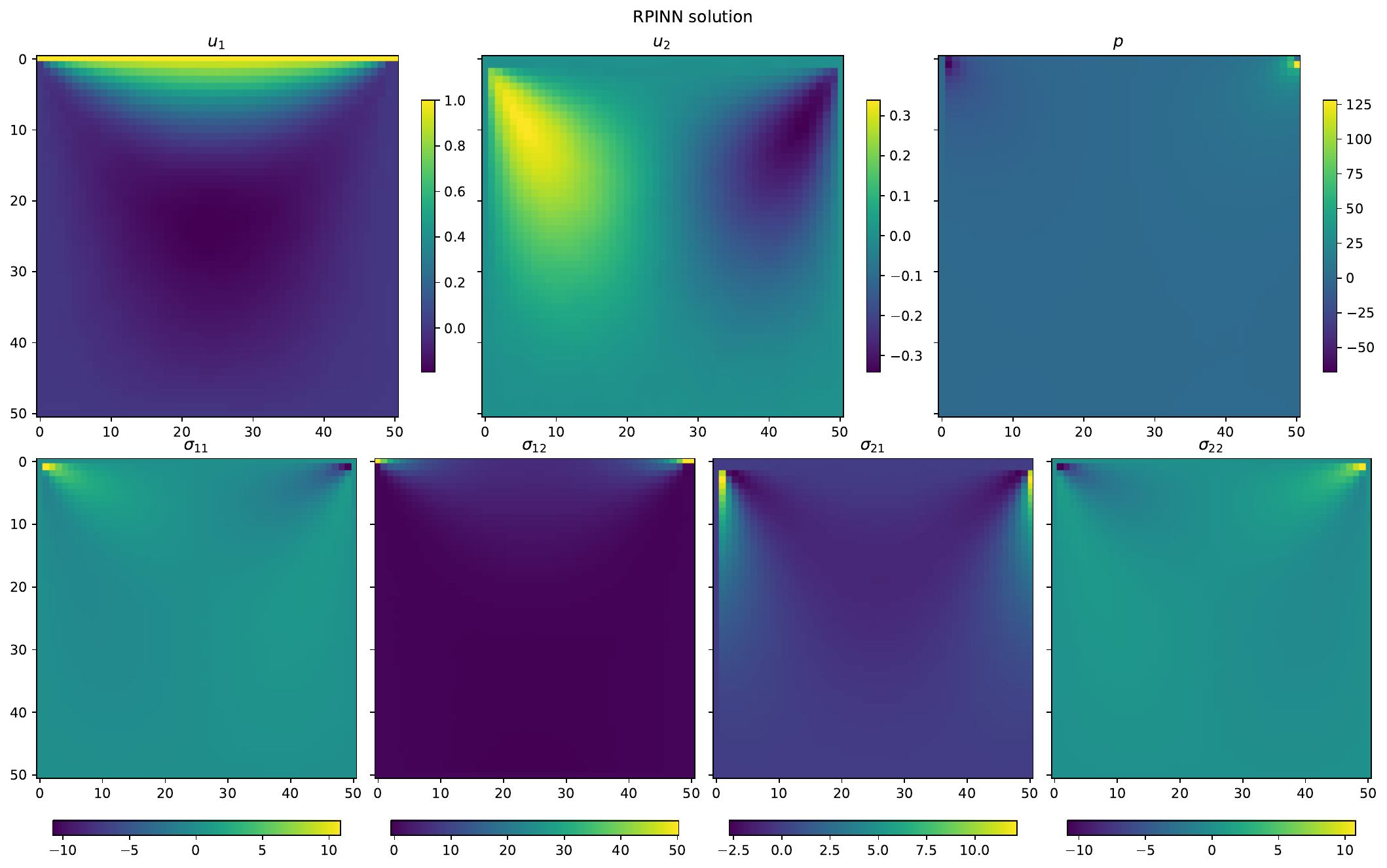}
    \caption{PINN solution for the cavity flow problem after 25,000 training epochs
    on a $50 \times 50$ grid.}
    \label{fig:trained-network}
\end{figure}

\begin{figure}
    \centering
    \includegraphics[width=0.7\linewidth]{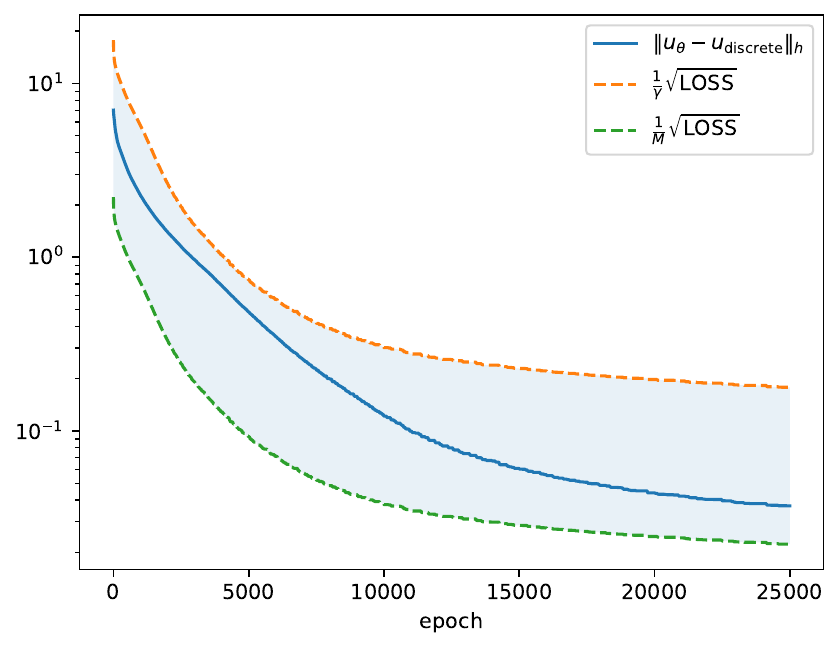}
    \caption{%
    Loss function value and true error of the best solution encountered so far
    throughout the training process
    for the cavity flow problem.}
    \label{fig:training-min}
\end{figure}

\newpage 

\section{Conclusions}

We propose a Python framework for solving PDEs
using discrete variational formulations based on the CRVPINN method. The CRVPINN method employs robust discrete loss function that is related with the true error.
Its goal is to facilitate building the matrices and right-hand side vectors
of such formulations by providing an intuitive, clean, concise, and convenient syntax
for defining weak variational forms,
inspired by FEniCS and the Unified Form Language.
It allows rapid prototyping of new formulations,
and integrates easily into a larger scientific computation ecosystem.
Its use alongside the PyTorch framework is illustrated on two PDEs --
a simple Laplace problem with uniform Dirichlet conditions,
and a more complex Stokes formulation,
showcasing tensor and composite function spaces.
In both cases, the mathematical formulation can be naturally translated
into code, even for relatively complex forms,
such as the ones encountered in the Stokes problem.
The numerical results demonstrate that the method
is capable of delivering accurate results
in a reasonable number of training epochs,
while providing a reliable and efficient error estimate in the form of its loss function.

\medskip
%
%
%

\noindent{\bf Acknowledgments} This work has been supported by the National Science Centre, Poland grant no. 
2025/57/B/ST6/00058. 
This work has received funding from the European Union's Horizon Europe research and innovation programme under the Marie Sklodowska-Curie grant agreement No 101119556.

\bibliographystyle{splncs04}
\bibliography{references}

\end{document}